\documentclass[letterpaper, 10 pt, conference]{ieeeconf}  

\IEEEoverridecommandlockouts                              

\usepackage{graphicx}
\usepackage{amsmath}
\usepackage{cite}
\usepackage{mathtools, amssymb}
\usepackage{varioref}
\labelformat{equation}{(#1)}
\usepackage{amsfonts}

\usepackage{color}
\usepackage{enumerate}

\usepackage{mysymbol}

\usepackage[dvipsnames]{xcolor}
\usepackage{tikz,pgfplots}
\usepackage{color} 
\usetikzlibrary{matrix} 
\usetikzlibrary{arrows} 
\usetikzlibrary{calc} 
\usetikzlibrary{shapes}
\pgfplotsset{compat=1.7}

\usetikzlibrary{arrows}
\usetikzlibrary{positioning}
\usetikzlibrary{backgrounds, fit}

\usepackage{amsthm}
\theoremstyle{plain}
\newtheorem{thm}{Theorem}

\theoremstyle{definition}

\newtheorem{remark}{Remark}
\newtheorem{example}{Example}

\renewcommand{\epsilon}{\varepsilon}

\usepackage{algorithm}
\usepackage{algpseudocode}
\usepackage{algorithmicx}

\usepackage{epstopdf}

\title{\LARGE 
\bf
Adaptive Scheduling for Machine Learning Tasks over Networks
}
\author{Konstantinos~Gatsis~
\thanks{
	The author is with the Department of Engineering Science, University of Oxford, Parks Road, Oxford, OX1 3PJ, UK. Email: konstantinos.gatsis@eng.ox.ac.uk.}%
}

\begin{document}

\maketitle
\thispagestyle{empty}
\pagestyle{empty}

\begin{abstract}
A key functionality of emerging connected autonomous systems 
such as smart transportation systems, smart cities, and the industrial 
Internet-of-Things, is the ability to process and learn from data 
collected at different physical locations. This is increasingly 
attracting attention under the terms of distributed learning and 
federated learning. However, in this setup data transfer 
takes place over communication resources that are shared 
among many users and tasks or subject to capacity constraints. This paper examines algorithms for 
efficiently allocating resources to linear regression tasks by exploiting the informativeness of the data. The algorithms developed enable adaptive scheduling of learning 
tasks with reliable performance guarantees.
\end{abstract}

\section{Introduction}

Conventional machine learning applications require data to be collected at a centralized location to be trained in a centralized manner, for example, running stochastic gradient descent from sampling the pool of data. However, the emergence of new cyber-physical architectures that are distributed requires rethinking this approach to enable learning across locations, for example when data are distributed or agents need to run tasks in different physical locations -- see, for example, Fig.~\ref{fig:architecture}. Applications domains of interest include, for example the Industrial Internet-of-Things, Smart Cities, and Smart Transportation Systems. 

One significant instance of this is the architecture of federated learning, a term used by Google \cite{konevcny2016federated, bonawitz2019towards}. The primary role of this architecture is to enable multiple users to jointly solve a machine learning problem over a communication network from data collected from the users. A major challenge is that in many modern machine learning applications data can be very high-dimensional, making their communication costly and inefficient. To alleviate this communication bottleneck, various approaches are being explored towards more communication-efficient machine learning. One direction is based on communicating the machine learning model parameters as they are being trained, such as the weights of a Deep Neural Network, instead of the data itseldf, or communicate the gradients of the optimization objective with respect to the parameters. In cases of deep learning models where even these parameters can be high dimensional neural network weights, sparsification and quantization of the weights or the gradients is also introduced to limit the communication cost~\cite{konevcny2016federated, aji2017sparse, sattler2019sparse, lin2020achieving}. Lazy updates are introduced in \cite{chen2018lag}, combinations of non-periodic updates and quantization is explored in \cite{reisizadeh2019fedpaq}, and considerations of federated multi-task learning is introduced in \cite{smith2017federated}.

Furthermore, when distributed learning is taking place over a wireless network, there is an interest in allocating the available network resources efficiently among the users holding the data~\cite{gunduz2019machine}. This has received increased interest in the wireless networking community leading to different directions including allocation of resources such as power~\cite{chen2019joint,dinh2019federated} or rates~\cite{ren2019accelerating,ahn2020cooperative, chang2020communication}. Often in such works the resource allocations are fixed and then machine learning performance properties such as convergence are analyzed. The problem of scheduling gradient updates over multiple access channels has also received initial consideration, for example comparing time-based approaches with approaches based on channel conditions~\cite{yang2019scheduling}, or including gradient information~\cite{amiri2020update, chen2020convergence}, or combinations with analog communication~\cite{amiri2020federated, sery2019analog}.
\begin{figure}[t!]
	\resizebox{\columnwidth}{!}{%
		\input{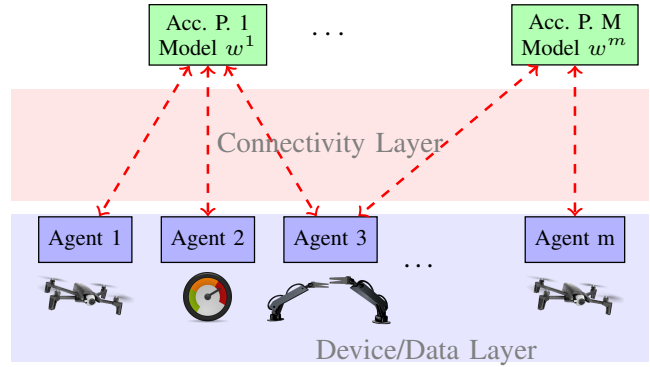}
	}
	\caption{{Architecture for solving machine learning tasks over communication networks of agents. Agents are collecting data and are communicating with access points/servers. Communication efficiency is addressed in this work by assessing the informativeness of the data.}}
	\label{fig:architecture}
\end{figure}

The purpose of this work is to introduce a new communication-efficient learning approach by
{posing a fundamental question: How informative are the collected data for the purpose of solving a machine learning problem? And how can this be used to communicate with efficiency over a network?}
Our approach hinges on the fact that when  model parameters need to be updated from noisy data, then not all updates are equally informative. As a result, performing updates selectively may be beneficial, and we can evaluate the informativeness of the data by estimating the gain in machine learning performance. Building upon this intuition we propose scheduling algorithms that aim to update the machine learning task whose data carry the most informative information, i.e., bring the most gain. By prioritizing updates  with more relevant information we aim to more efficiently use the communication resources.
This is illustrated when communication capacity is limited, and it is undesirable to have a node sending updates all the time, or when {multiple machine learning tasks} are solved over a network and not all tasks can be updated from distributed data at the same time.

The approach is developed for the problem of solving linear regression tasks in order to minimize the expected square prediction errors for the tasks. The problem is formulated in Section~\ref{sec:setup} and the scheduling problems to deal with communication capacity constraints are introduced.
In Section~\ref{sec:approach} we present our scheduling algorithms, which attempt to prioritize updates whose data carries the most information, i.e., that would lead to the highest performance gain (minimize the prediction loss). 
The approach is theoretically analyzed and  guarantees are provided about convergence and the tradeoff with required communication resources.
Furthermore, as our method is ideally developed when the loss function is completely known, i.e., the data distribution is known, significant effort is placed on an approximate scheduling scheme that uses the currently available data to estimate how informative the current update will be.

Our technical methodology borrows ideas from the problem of scheduling multiple control tasks over a shared resource, where approaches such as opportunistic scheduling over wireless links \cite{EisenEtal19a,GatsisEtal15} and age-of-information or value-of-information \cite{mamduhi2017error, ayan2019age, soleymani2016optimal} are proposed. Our methodology is also related to approaches in event-triggered learning that try to update only if necessary~\cite{solowjow2018event, zhao2020event}.
Moreover, the problem of solving multiple tasks concurrently is studied, unlike previous approaches that are focused on solving a single machine learning task. This creates both a higher communication load, but also an opportunity to differentiate among tasks based on their performance gain, as described above. 

In numerical evaluations in Section~\ref{sec:numerical} we validate the performance of our algorithms compared to other scheduling benchmarks and evaluate the performance gains at different regimes, when the covariance of the data changes, or when the number of available data is small or large. 
Our evaluations show significant performance improvements compared to other approaches in the literature that treat the magnitude of the gradients as a measure of the informativeness of the current update. We conclude with a discussion on future work.

\section{Problem Setup}\label{sec:setup}

{
	
	The architecture examined in this paper, shown in Fig.~\ref{fig:architecture}, involves multiple access-points/servers interested in building data-driven models by solving machine learning tasks on data that are collected by multiple agents. 
	In general, both multiple machine learning tasks may need to be solved, and multiple agents are involved. 
	We consider machine learning tasks indexed by $j=1, \ldots, m$. For each task we are interested in finding a vector of weights (parameters) $w^j$ of appropriate dimensions. To assess the performance of each task we also have a performance metric (cost) to be minimized denoted by $J^j(w^j)$ for each task.
	
	The aim will be to achieve this with communication efficiency. 
	This is for example the case when an agent should not communicate all the time over a communication network to update the vector of parameters at the access point/server. Or when due to capacity constraints not all agents can communicate at the same time. This leads to problems of scheduling the data updates.
Below we give a more precise description of the learning tasks in our framework, which correspond to linear regression tasks, and the scheduling problems considered.

}

\subsection{Model for each learning task}\label{sec:task}

We consider each machine learning task to be a linear regression problem~\cite[Ch. 9]{shalev2014understanding}. 
We are interested in finding a vector of weights $w$ that explains the relationship between random variables $(x,y)\in \reals^n \times \reals$, i.e., $y \approx x^T w$. The random variables $(x,y)$ follow in general a joint distribution denoted by $\mu$. The desired choice for the weights is the one that minimizes the expected square prediction error, i.e.,
\begin{equation}\label{eq:regression_problem}
\min_{w} J(w) = \frac{1}{2} \mathbb{E}_{(x,y)\sim \mu}(y - x^T w)^2 
\end{equation}
where the expectation is with respect to the data distribution $\mu$ -- in the sequel we drop this notation when it is implied that expectation is with respect to this distribution.

The optimal solution $w^*$ is given as the solution to the linear equations 
\begin{equation}
	\mathbb{E}xx^T w^* - \mathbb{E}xy = 0. 
\end{equation}
Towards finding an optimal set of weights, we would like to employ a gradient descent algorithm. Starting from some initial set of weights $w_0$ we would like to update the weights according to
\begin{equation}
w_{k+1} = w_k - \epsilon \nabla J(w_k)  
\end{equation}
where $\nabla J(w_k) = \mathbb{E}xx^T w_k - \mathbb{E}xy$, and $\epsilon >0$ is a small positive stepsize. As will be illustrated later, choosing $\epsilon<2/\lambda_{\max}(\mathbb{E}xx^T)$ guarantees convergence.

The distribution of the data is not a priori known, and hence as is common in machine leaning, e.g., in empirical risk minimization~\cite[Ch. 2]{shalev2014understanding}, we will attempt to minimize the empirical cost computed as an average over collected data. {Specifically we assume that at each iteration $k$ there are $N$ new data points 
of the form 
\begin{equation}\label{eq:data}
(x_i, y_i) \in \reals^n \times \reals, \qquad i =1, \ldots, N.
\end{equation} 
We assume each data pair is independent and identically distributed according to a distribution $\mu$.\footnote{This setup arises either when an agent in Figure~\ref{fig:architecture} collects N new independent samples at each iteration, or when it just maintains a large pool of samples and selects randomly $N$ from them at each iteration as frequently done in stochastic gradient descent practice.}} 
Then we form the empirical cost
\begin{equation}
	\hat{J}(w) = \frac{1}{2} \frac{1}{N} \sum_{i=1}^N (y_i - x_i^T w)^2
\end{equation}
{It may also be desirable to add regularization terms in the objective to avoid overfitting, but in this paper we restrict attention to the objective above.}
With this approximation, 
we follow a stochastic gradient vector 
\begin{equation}
w_{k+1} = w_k  - \epsilon g_k
\end{equation}
computed over the data as
\begin{equation}\label{eq:gradient_estimate}
 g_k= \nabla\hat{J}(w_k) = \frac{1}{N}\sum_{i=1}^N \left( x_i x_i^T w_k - x_i y_i\right)
\end{equation}
After this update the prediction error becomes
\begin{equation}
J(w_{k+1}) = \frac{1}{2} \mathbb{E}(y - x^T w_{k+1})^2
\end{equation}
where the expectation is with respect to the distribution $\mu$.

We note that since the $N$ data points are random, so is the constructed gradient direction $g_k$, the updated vector $w_{k+1}$, as well as the performance metric $J(w_{k+1})$. 
{To evaluate how good is this updated prediction error, we would like to measure
	\begin{equation}
	\mathbb{E} [J(w_{k+1}) | w_k] = 
	\mathbb{E}_{data \sim \mu^N} [J(w_{k+1}) | w_k]
	\end{equation} 
It is important to note that the expectation here is over the $N$ i.i.d. data that are collected at iteration $k$ and used to construct the stochastic gradient $g_k$ and update the iterate $w_{k+1}$. In the paper, whenever an expectation over iterates $w_k$ is taken, this is an expectation over all the data collected until time $k$.}

\begin{figure}[t!]
	\includegraphics[width=\columnwidth]{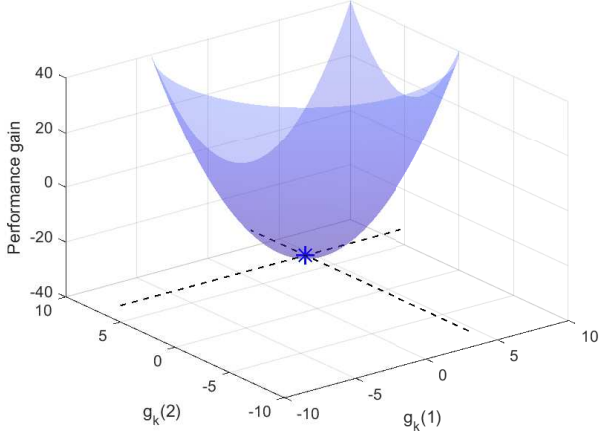}
	\caption{Example of the performance gain as a function of the stochastic gradient update in two dimensions ($n=2$). Depending on the informativeness of the data, the gain varies and is exploited in the proposed scheduling mechanism in Section~\ref{sec:approach}.}
	\label{fig:gain_example}
\end{figure}

\subsection{Scheduling problem for single task}

Given the above modeling for a machine learning task that needs to be solved, the scheduling problem is as follows: given the stochastic gradient $g_k$ available from the data at each iteration, select whether to transmit this gradient update over the communication network to a receiving server. The server maintains a current vector of weights $w_k$ which will be updated depending on the case to
\begin{equation}\label{eq:dynamics_single_task}
w_{k+1} = \left\{ \begin{array}{ll} w_k - \epsilon g_k &\text{if scheduled}\\
w_k &\text{if not scheduled} \end{array}\right.
\end{equation}
We also denote with $\alpha_k \in \{1,0\}$ the decision to schedule/transmit or not. {This decision is made at the node.}
{At the next iteration $k+1$ a new set of data is collected as in \ref{eq:data}, a new stochastic gradient direction $g_{k+1}$ is computed, and the process repeats. The aim will be to avoid sending updates all the time to limit the communication burden.}

\subsection{Scheduling problem for multiple tasks}

Given the above modeling for each of the machine learning tasks that need to be solved, the scheduling problem for multiple tasks is as follows. Given the available data across all tasks $j=1, \ldots, m$, select one task $j^*$ whose weights will be updated at the server changing values from $w_k^{j^*}$ to $w_{k+1}^{j^*}$, while all other weights remain the same at $w_k^j$ for all $j \neq j^*$. More succinctly, for any given scheduling rule, we denote the vector after the scheduling decision as 
\begin{equation}\label{eq:dynamics_multiple_tasks}
	w_{k+1}^j = \left\{ \begin{array}{ll} w_k^j - \epsilon g_k^j &\text{if $j$ scheduled}\\
	w_k^j &\text{if $j$ not scheduled} \end{array}\right.
\end{equation}
{
This means that in an architecture as shown in Fig~\ref{fig:architecture}, due to communication capacity constraints, at most one communication link is active at any time step, and hence at most one of the tasks can be updated over the network. The setup can be easily generalized to the case where $p$ out of $m$ tasks can be scheduled, for some $p<m$.
}

\section{Proposed scheduling approach}\label{sec:approach}

\subsection{Scheduling a single task}
{ The approach is based on the notion of performance gain which can be thought as a measure of how informative are the data collected at each time step with respect to the machine learning problem. The gain can be calculated as the difference
\begin{equation}
	J(w_k - \epsilon g_k) - J(w_k),
\end{equation}
measuring how much will the objective change after the update. Owing to the fact that the objective function is a quadratic function of the weights $w$, this gain can be thought as a quadratic function of the stochastic gradient $g_k$. The best gain (minimum) is a negative value and is achieved at some vector $g_k$ that updates the weights to the optimal solution of the problem. On the other hand, as $g_k$ grows larger in any direction, the quadratic form takes a larger positive value, meaning that actually no gain is achieved. The update direction in turn depends on the random data samples by \ref{eq:gradient_estimate}. As a result the gain can be large or small depending on the\textit{ informativeness of the data} and how well they point towards to the optimal solution. The gain also depends on the current vector of weights $w_k$. An illustration is shown in Fig.~\ref{fig:gain_example}.

The proposed approach then is to send a gradient update if the gain is large enough. Mathematically we write
\begin{equation}\label{eq:single_scheduling}
	\alpha_k = \left\{ \begin{array}{ll} 1 &\text{if } J(w_k - \epsilon g_k) - J(w_k) \leq -\lambda \\
	0 &\text{otherwise}\end{array}\right.
\end{equation}
for some scalar parameter $\lambda>0$.

Intuitively this approach saves up communication resources, because the larger the parameter value $\lambda$ is, the more infrequent the updates will be. But then the question is what can be said about the progress of learning. We have then the following result.

\begin{thm}[Convergence]\label{thm:theorem_single}
	Consider the optimization problem defined in \ref{eq:regression_problem}. Consider the update rule in \ref{eq:dynamics_single_task}. Suppose $g_k$ are independent random variables with mean equal to $\nabla J(w_k)$ at each iteration $k$ and covariance $G$. Consider the update rule in \ref{eq:single_scheduling}. Then for any iteration $N$ we have that 
	\begin{equation}\label{eq:single_statement}
		\mathbb{E}J(w_N) \leq \rho^N J(w_0) + 
		(1-\rho^N) \left[ J(w^*) + \frac{\lambda + \epsilon^2 \text{Tr}(\Sigma_x  G)}{1-\rho} \right]
	\end{equation}
	where the expectation is with respect to the data collected until iteration $N$, and the parameters  are $\Sigma_x = \mathbb{E}xx^T/2$ and $\rho = \max_i (1-\epsilon  \lambda_i(\mathbb{E}xx^T))^2$. 
\end{thm}

\begin{proof}
	First we note that due to the choice in \ref{eq:single_scheduling} the following inequality holds for all times (technically it holds almost surely as all the variables involved are random variables)
	\begin{equation}
		\alpha_k J(w_k - \epsilon g_k) + (1-\alpha_k) J(w_k) \leq \lambda + J(w_k - \epsilon g_k).
	\end{equation}
	This can be easily verified by examining the two cases $\alpha_k =0$ or $1$.
	Then note that by the dynamics in \ref{eq:dynamics_single_task} we can rewrite the left hand side as
	\begin{equation}
	J(w_{k+1}) \leq \lambda + J(w_k - \epsilon g_k).
	\end{equation}
	Taking expectation over the stochastic gradient $g_k$, conditioned on the current iterate $w_k$, we get that 
	\begin{equation}\label{eq:intermedient}
		\mathbb{E}[ 
		J(w_{k+1}) \given w_k] \leq \lambda + \mathbb{E}[ J(w_k - \epsilon g_k)\given w_k]
	\end{equation}
	
	Then given the fact that the function $J(w)$ is quadratic and can be rewritten as $J(w) = J(w^*) + (w-w^*)^T \Sigma_x (w-w^*) $, and the properties of the stochastic gradient, we have that 
	\begin{align}
		&\mathbb{E}[ J(w_k - \epsilon g_k)\given w_k] = J(w^*)+ \notag\\
		&(w_k - \epsilon \mathbb{E} g_k-w^*)^T \Sigma_x (w_k - \epsilon \mathbb{E} g_k-w^*) + \text{Tr}(\Sigma_x G) 
	\end{align}
	Here we can substitute $\mathbb{E} g_k = \nabla J(w_k) = 2 \Sigma_x (w_k-w^*)$, and the fact that $(I-\epsilon2\Sigma_x)'\Sigma_x (I-\epsilon2\Sigma_x)\preceq \rho \Sigma_x$ to conclude that
	\begin{align}
\mathbb{E}[ J(w_k - \epsilon g_k)&\given w_k] \leq J(w^*)+ \notag\\
	&\rho (w_k - w^*)^T \Sigma_x (w_k -w^*) + \text{Tr}(\Sigma_x G) 
	\end{align}
	Substituting this in \ref{eq:intermedient} and rearranging terms we find that 
	\begin{align}
	\mathbb{E}[ 
	J(w_{k+1}) \given w_k] &\leq \lambda + \rho J(w_k) \notag\\
	&+ \epsilon^2 \text{Tr}(\Sigma_x  G) + (1-\rho)J(w^*)
	\end{align}
	Taking expectation on both sides with respect to the variable $w_k$, and iterating over time $k=1, \ldots, N$, we get the desired result \ref{eq:single_statement}.
\end{proof}

The result verifies that the update rule converges (in a stochastic sense) because $\rho<1$ as can be confirmed by the appropriate choice of the stepsize $0<\epsilon<2/\lambda_{\max}(\mathbb{E}xx^T)$. Essentially the result follows because the function $J(w)$ can be thought as a Lyapunov function for the stochastic dynamics of the update in \ref{eq:dynamics_single_task} and showing the convergence result. A direct consequence of the above result is 
\begin{equation}
\limsup_{N \rightarrow \infty} \, \mathbb{E}J(w_N) \leq  J(w^*) + \frac{\lambda + \epsilon^2 \text{Tr}(\Sigma_x  G)}{1-\rho} 
\end{equation}
This means that eventually we get  close to the optimal set of weights $w^*$ subject to some overshoots. There is the effect of the stochastic gradient and its covariance $G$, which can be made small in practice by choosing the step size $\epsilon$ to be small -- or by choosing a diminishing stepsize which will be analyzed in future work. Moreover, there is a penalty proportional to the parameter $\lambda$, introduced to save up on communication cost.

\begin{remark}
	In Theorem~\ref{thm:theorem_single} we assumed for simplicity that the stochastic gradients have bounded covariances that are constant over time. In reality for the problem above the covariance of the stochastic gradient in \ref{eq:gradient_estimate} will depend on the current iterate $w_k$, but our choice can be justified in two ways. We can either consider these covariances to be uniformly bounded over time by some constant $G$. Or alternative if we consider the case close enough to the equilibrium $w_k \approx w^*$, then it follows that the covariances are indeed constant over time. A more detailed investigation will be explored in a follow up work.
\end{remark}

{ Furthermore, we can establish the following guarantee about the total communication rate of the proposed approach. 
	
\begin{thm}[Communication guarantee]
	Consider the same setup as in Theorem~\ref{thm:theorem_single}.
		The total expected communication rate satisfies
		\begin{equation}\label{eq:communication_guarantee}
		\limsup_{N \rightarrow \infty} \sum_{k=0}^N \mathbb{E} \alpha_k 
		\leq\frac{ J(w_0) - J(w^*)}{\lambda }.
		\end{equation}
		Here the expectation is with respect to the data collected as iterations $N\rightarrow \infty$.
\end{thm}
}
\begin{proof}
	
	{ 
		Due to the choice in \ref{eq:single_scheduling} the following inequality holds for all times (technically it holds almost surely as all the variables involved are random variables)
		\begin{equation}
		\lambda \alpha_k + 
		J(w_{k+1}) \leq J(w_k).
		\end{equation}
		This can be easily verified by examining the two cases $\alpha_k =0$ or $1$. Taking expectation, iterating over time $k=0, \ldots, N$, and summing up, we conclude that
		\begin{equation}
		\lambda \sum_{k=0}^N \mathbb{E} \alpha_k + 
		\mathbb{E}J(w_{N+1}) \leq J(w_0).
		\end{equation}
		This can be rearranged as
		\begin{equation}\label{eq:communication_guarantee_1}
		\sum_{k=0}^N \mathbb{E} \alpha_k 
		\leq\frac{ J(w_0) - \mathbb{E}J(w_{N+1})}{\lambda }
		\end{equation}
		Moreover, since any value of the variable $w_{N+1}$ is in general suboptimal, we have that $\mathbb{E}J(w_{N+1})\geq J(w^*)$. From which we get the desired result \ref{eq:communication_guarantee}. 
	}
\end{proof}

This verifies that increasing $\lambda$ will decrease the resulting communication rounds in an inversely proportional manner. Other schemes are also possible, such as decreasing $\lambda$ at each round, and will be considered in subsequent work.

\subsection{Practical scheduling scheme}

Despite the above guarantee, implementing the proposed scheduling scheme in \ref{eq:single_scheduling} would be practically challenging because it requires information that is not known. Specifically it would require knowledge of the data distribution in order to compute the actual performance gain. Since the true distribution is unknown, one approach is to \textit{estimate the performance gain from the data}. In particular, since the objective function is quadratic, we can write the performance gain as
\begin{align}\label{eq:gain_1}
J(w_{k}- \epsilon g_k) - J(w_k) = &-\epsilon g_k^T \nabla J(w_k) \notag\\
&+ \frac{1}{2} \epsilon^2 g_k^T \nabla^2 J(w_k) g_k
\end{align}
Then we can approximate the quantities
\begin{align}
&\nabla J(w_k) \approx \frac{1}{N}\sum_{i=1}^N \left( x_i x_i^T w_k - x_i y_i\right) =g_k \\
&\nabla^2 J(w_k) \approx \frac{1}{N}\sum_{i=1}^N x_i x_i^T
\end{align}
where we note that the stochastic gradient direction $g_k$ appears again. Hence, using the expression for the information gain in \ref{eq:gain_1}, we can approximate the gain as
%
\begin{align}\label{eq:approximate_gain_1}
J(w_{k}- \epsilon g_k) - J(w_k) \approx -\epsilon g_k^T\left[ I - \epsilon \frac{1}{2}  \frac{1}{N}\sum_{i=1}^N x_i x_i^T \right] g_k
\end{align}
It is crucial to emphasize that \textit{this is no longer a simple quadratic function} of the data but a more complicated function - we note that the data appear both in the stochastic gradients $g_k$ by \ref{eq:gradient_estimate} as well as the matrix in the middle. An example is presented next. This approximate value of the gain may take again positive or negative values but it induces an approximation error/bias.

As a result, we can implement the scheduling decision in \ref{eq:single_scheduling} with the approximation in \ref{eq:approximate_gain_1}. In this case we no longer have the performance guarantee in Theorem~\ref{thm:theorem_single}. In numerical evaluations however we see that despite the bias this mechanism performs very well.

\begin{example}
	Consider as an illustration the case where a task has only one $N=1$ sample $(x_1, y_1)$ and the initial set of weights is zero $w_0 = 0$. Then the stochastic gradient step by \ref{eq:gradient_estimate} is $g_k = - x_1 y_1$ and the estimated gain by \ref{eq:approximate_gain_1} is 
	\begin{align}
	J(w_1) - J(0) \approx -\epsilon y_1^2 \|x_1\|^2 ( 1 - \frac{\epsilon}{2} \|x_1\|^2)
	\end{align}
	This estimated gain is negative for small values of the norm $\|x_1\|$ and attains its minimum value at the point $\|x_1\|=1/\sqrt{\epsilon}$. The estimated gain increases and becomes positive as $\|x_1\|$ increases. This means that our scheduling mechanism in \ref{eq:single_scheduling} would avoid scheduling the updates in that case. The expression becomes more complex for larger number of samples.
\end{example}

}

\subsection{Scheduling multiple tasks}

{
 We have $m$ tasks, with objectives $J^1(w^1), \ldots ,J^m(w^m)$ and we are interested in selecting to update one of the vectors $w^j$. We propose a greedy scheduling algorithm: \textit{select the update that will bring the best gain}. Using the notation in the above section we select 
\begin{equation}\label{eq:scheduling}
	j^* = \argmin_{j=1, \ldots,m} J^j(w^j_k- \epsilon^j g_k^j) - J^j(w^j_k)
\end{equation} 
where the gain can be approximated from the data for each task as explained earlier in \ref{eq:approximate_gain_1}.  This can exploit the fact that not all of the updates are equally informative due to the randomness of the data collected for each task.
}

The theoretical guarantees for this scheduling approach will be further developed in future work. The practical benefits will be shown numerically in the next section.

\begin{remark}[Practical Implementation]
	{The scheduling decision in \ref{eq:scheduling} is made in a centralized manner. This can be made for example at the server end of the architecture in Fig.~\ref{fig:architecture}, deciding which node to upload information, or at a node if it needs to send updates to many different servers. The proposed approach requires having access to all the available data when making the scheduling decision, and compare the informativeness of the data per task in \ref{eq:scheduling}. This is feasible in the second scenario at the node, but is more challenging to implement in practice in the first scenario at the server end. One potential approach, also proposed in \cite{amiri2020update,chen2020convergence}, is to perform this in two stages: first have each node estimate the performance gains locally from the data, and only send these scalar gains at the server (not the data or the gradients). Then the server can compare the gains among them and decide which gradient vector to receive. Other approaches, such as estimating the performance gain at the server end, will be examined in future work.}
\end{remark}

\begin{remark}[Computational cost]
	We note here that the computational cost for making the scheduling decision in \ref{eq:scheduling} is small. First, for each task, the vector $g_k$ given by \ref{eq:gradient_estimate} can be computed by first computing the residual terms $x_i^T w_k - y_i$ and then multiplying them with the vectors $x_i$ and summing up for $i=1,\ldots, N$. Then, for each task, the  term by \ref{eq:approximate_gain_1}  can be computed as
	\begin{equation}
		  \epsilon\frac{1}{2}  \frac{1}{N}\sum_{i=1}^N (g_k^T x_i)^2 -g_k^T g_k
	\end{equation}
	Overall across tasks these require $O(Nnm)$ operations. Finally finding the minimum across $m$ in \ref{eq:scheduling_gradient} is not costly for a small number of tasks $m$.
\end{remark}

\subsection{Other approaches in the literature: Gradient-based scheduling}

We note that a different scheduling approach would be to treat the tasks with the largest updates as the most important. In that case, a simple approach would be to schedule the task with the largest norm of the (stochastic) gradient, i.e., 
\begin{equation}\label{eq:scheduling_gradient}
j^* = \argmax_{j=1, \ldots,m} \|g_k^j\|^2
\end{equation} 
In numerical comparisons we show that this scheme often leads to worse performance. From our expression on \ref{eq:approximate_gain_1} we see that for small stepsizes $\epsilon$ the magnitude of the gradient may serve as a proxy for the performance gain. This may also be the case when the Hessian of the problem is closer to an identity matrix.

We note that the idea of scheduling based on gradient magnitudes has been proposed in very recent works in federated learning over wireless channels~\cite{amiri2020update, chen2020convergence}. Similarly, in the context of sparsification and quantization for high-dimensional gradient updates, it has been proposed that elements with larger values should be given priority~\cite{aji2017sparse, sattler2019sparse}. Our findings hence point to a \textit{novel and more communication- efficient approach for gradient updates}.

\section{Numerical results}\label{sec:numerical}

In this section we analyze and compare numerically the proposed approach. In this section we make an additional assumption about the data samples. Specifically we assume that the points $x_i$ are i.i.d. Gaussian random variables, while the points $y_i$ are given as $y_i = x_i^T w^* + \eta_i$ where $w^*$ is the true parameter and $\eta_i$ are i.i.d. Gaussian measurement noises.

\subsection{Tradeoff between communication rate and learning performance}

\begin{figure}
	\includegraphics[width=\columnwidth]{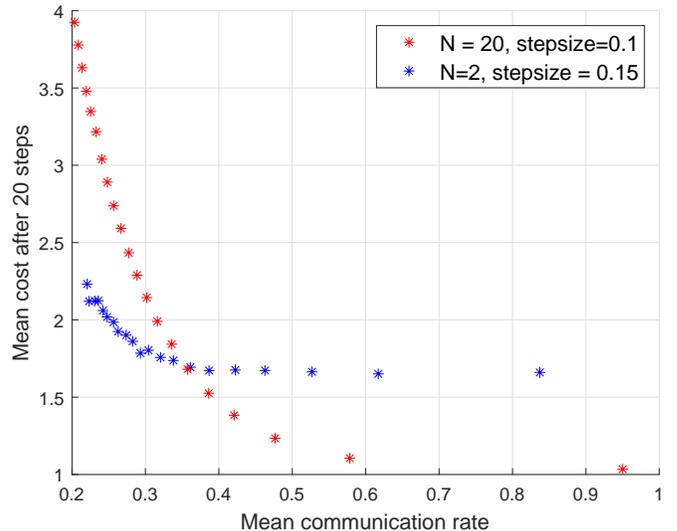}
	\caption{Evaluation of the tradeoff between communication rate and machine learning performance of the proposed scheduling approach in \ref{eq:single_scheduling}. The curves are obtained by varying the values of the parameter $\lambda$ for different settings of the optimization problem.}
	\label{fig:tradeoff}
\end{figure}

{We consider the scheduling algorithm in \ref{eq:single_scheduling} with the performance gains estimated as in \ref{eq:approximate_gain_1}. We consider a problem with dimensions $n=2$, with covariances $\mathbb{E} x x^T = \left[ \begin{array}{cc} 3 &0\\ 0 &1 \end{array}\right]$ (which affects the Hessian of the problem), the initial weights are $w_0=0$, and the true weights equal to $w^* = \left[ \begin{array}{c} 3 \\ 5 \end{array}\right]$. First for stepsize $\epsilon = 0.1$ and $N=20$ data points available at each iteration (cf.\ref{eq:data}), we simulate algorithm \ref{eq:single_scheduling} for varying values of the parameter $\lambda$. In Fig.~\ref{fig:tradeoff} we plot the observed mean learning performance after $20$ iterations ($J(w_{20})$) versus the average communication rate ($1/20 \sum_{k=0}^{19} \alpha_k$). We observe that the proposed scheduling approach indeed allows us to tradeoff communication rate with machine learning performance.
	
Then for stepsize $\epsilon = 0.15$ and $N=2$ data points available at each iteration, we simulate the algorithm again for varying value of the parameter $\lambda$ and plot the results in Fig.~\ref{fig:tradeoff}. This case corresponds to both a very noisy gradient setting due to the very small number of data at each iteration and also a larger stepsize. We observe now qualitatively a different tradeoff curve. In this case we can remarkably lower the communication rate without a significant loss in learning performance. This indicates that our proposed scheduling approach may achieve better communication efficiency when data are more noisy.

}

\subsection{Bias of data-based estimated performance gains}

\begin{figure}
	\includegraphics[width=\columnwidth]{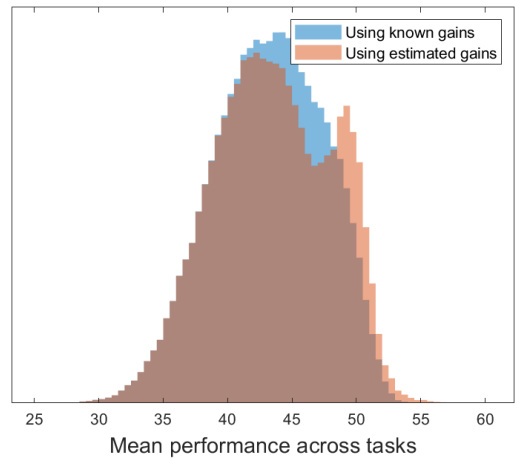}
	\caption{Comparison between our scheduling approach in \ref{eq:scheduling} knowing the distributions to compute the gains by \ref{eq:gain_1} versus estimating the gain by \ref{eq:approximate_gain_1}. We observe that the lack of model knowledge is counteracted from available data.}
	\label{fig:known_vs_estimated}
\end{figure}

Next we consider the problem of scheduling multiple $m$ tasks, using the rule in \ref{eq:scheduling}. We would like to investigate how much bias is introduced by our practical scheme that is based on estimating the performance gain for each task update based on the currently available data. Hence we compare \ref{eq:scheduling} when using the performance gains computed by \ref{eq:gain_1} that requires knowledge of the problem distributions, with the fully data-based scheme in \ref{eq:approximate_gain_1}. 

In particular we consider scheduling $m=2$ linear regression tasks with randomly chosen parameters. While all problem parameters are fixed, we consider a single iteration of the problem, and we draw many data sample batches of size $N$ and simulate the two scheduling approaches. In Fig.~\ref{fig:known_vs_estimated} we plot the histogram of the mean performance across tasks, computed as $\frac{J^1(w_{k+1}^1) + J^2(w_{k+1}^2)}{2} $ for the two schemes. In our numerical evaluations, we surprisingly do not observe a significant loss due to the estimation procedure. This was consistently observed across different instances, reinforcing the usefulness of our scheme. In particular we observe that the difference between the two schemes is larger when the coordinates of $x$ are more dependent. Our intuition is that in that case the covariances $\mathbb{E} x x^T$ are more different than identity matrices and it is harder to estimate the performance gains  using data. The spike that appears in Fig.~\ref{fig:known_vs_estimated} is because of the multi-modal nature of the problem; sometimes the data-based approach has errors which lead to schedule tasks that lead to larger costs.

\subsection{Comparison with gradient-based scheduling}

\begin{figure}
	\includegraphics[width=\columnwidth]{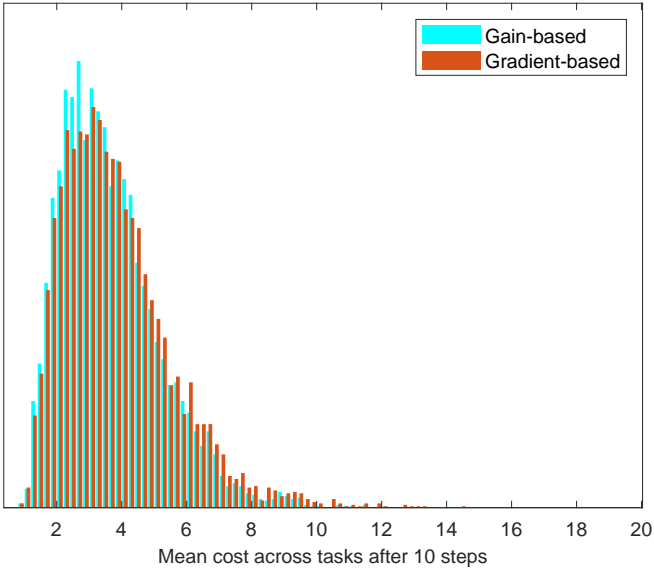}
	\caption{Comparison between our scheduling approach based on estimating the gain in \ref{eq:approximate_gain_1} versus the approach in \ref{eq:scheduling_gradient} based on the magnitude of the gradient.}
	\label{fig:ours_vs_gradient01}
\end{figure}

\begin{figure}
	\includegraphics[width=\columnwidth]{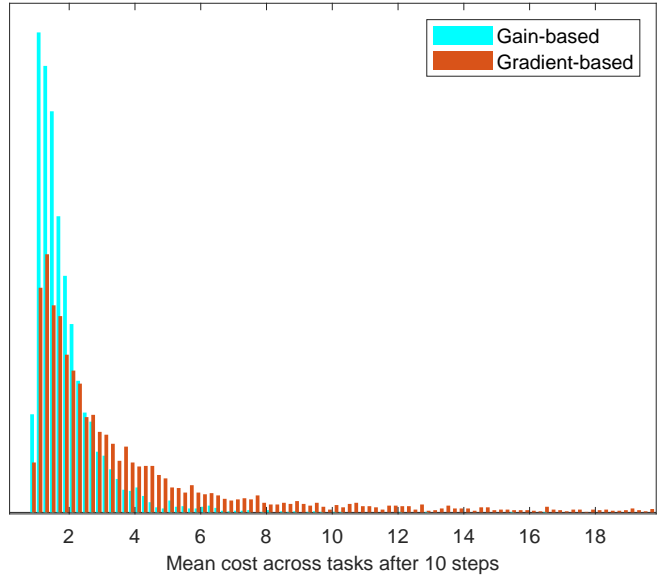}
	\caption{Comparison between our scheduling approach based on estimating the gain in \ref{eq:approximate_gain_1} versus the approach in \ref{eq:scheduling_gradient} based on the magnitude of the gradient. The gradient-based approach leads to considerably worse performance when the stepsize is large.}
	\label{fig:ours_vs_gradient02}
\end{figure}

We finally compare our scheduling scheme \ref{eq:scheduling} based on estimating the performance gain across tasks in \ref{eq:approximate_gain_1} with the simple scheduling strategy based on the magnitude of the gradients for each task \ref{eq:scheduling_gradient}. 
To illustrate how these approaches can lead to very different results we consider a specific setup of scheduling $m=2$ identical tasks of dimensions $n=2$, with covariances $\mathbb{E} x x^T = \left[ \begin{array}{cc} 3 &0\\ 0 &1 \end{array}\right]$ (which affects the Hessian of the problem), the initial weights are $w_0=0$, and the true weights equal to $w^* = \left[ \begin{array}{c} 3 \\ 5 \end{array}\right]$. We assume $N=5$ data points are available at each iteration per task. First for stepsize $\epsilon = 0.1$ the comparisons are shown in Fig.~\ref{fig:ours_vs_gradient01} again as histograms of the mean cost across the two tasks, where we observe very close performance among both approaches - even though the proposed gain-based one is better.
Both approaches of course have the same communication burden - only one of the two tasks is selected to communicate/update at each step.
Then for a slightly larger stepsize $\epsilon = 0.2$ the comparisons are shown in Fig.~\ref{fig:ours_vs_gradient02}.
We observe that our approach performs significantly better than the gradient-based one. The improvements get even more significant as the setpsize increases. Our conclusion is that \textit{the magnitude of the gradient is not a reliable measure for the informativeness of the data}. Our approach which is based on the more complex estimate of performance gain provides a more reliable and communication-efficient approach.

\section{Concluding remarks}

In this paper we examine the problem of solving multiple machine learning tasks concurrently over a network. We consider the problem of selecting which updates to be scheduled to meet capacity constraints. To exploit the informativeness of the data we examine the notion of performance gain and we illustrate numerically how this can be approximated from the data without further model knowledge. The approach is contrasted to other related works in the area of communication-efficient learning. 

A limitation of our approach is that scheduling must be performed centrally with all data points available. In practical scenarios, such as in federated learning, the data are distributed across different agents and hence alternative approaches will be explored in future work.
Ongoing work explores the use of the approach in more complex networks of learning agents, as well as other machine learning tasks beyond linear regression.

\bibliographystyle{ieeetr}
\bibliography{federated_learning}

\end{document}